%% file: neurips_2025.tex
\documentclass{article}

\PassOptionsToPackage{numbers, compress}{natbib}
\usepackage[preprint]{neurips_2025}

 \usepackage{amsmath}
\usepackage[utf8]{inputenc} 
\usepackage[T1]{fontenc}    
\usepackage{hyperref}       
\usepackage{url}            
\usepackage{booktabs}       
\usepackage{amsfonts}       
\usepackage{nicefrac}       
\usepackage{microtype}      
\usepackage{xcolor}         
\usepackage{xpatch}
\usepackage{amsthm} 
\newtheorem{theorem}{Theorem} 
\usepackage{xpatch}
\usepackage{amsthm}
\usepackage{mdframed} 
\usepackage{graphicx}
\usepackage{algorithm}
\usepackage{algpseudocode}
\usepackage{mathtools}
\usepackage{wrapfig}

\usepackage{tikz} 
\usepackage[utf8]{inputenc} 
\usepackage[T1]{fontenc}    
\usepackage{url}            
\usepackage{microtype}      
\usepackage{algorithm}
\usepackage{algpseudocode}
\usepackage{listings}
\usepackage{overpic}
\usepackage{wrapfig}
\usepackage{colortbl}
\usepackage{tabularx}
\usepackage{verbatim}
\usepackage{nicefrac}
\usepackage{empheq}         

\usepackage{booktabs}   
\usepackage{caption} 
\usepackage{subcaption}
\usepackage{multirow}       
\usepackage{makecell}       
\usepackage{adjustbox}      
\usepackage{graphicx}    
\usepackage{array}

\usepackage{tikz}
\usepackage{amsmath} 

\usepackage{centernot}
\usepackage{xcolor}
\definecolor{citecolor}{HTML}{0071BC}
\definecolor{linkcolor}{HTML}{ED1C24}

\usetikzlibrary{arrows.meta, positioning, calc}

\definecolor{myorange}{RGB}{255,127,0} 
\definecolor{myblue}{RGB}{0,114,189}  

\usepackage{amsthm}
\usepackage{amssymb}
\usepackage{amsmath}
\usepackage{amsfonts}       
\usepackage{nicefrac}       
\usepackage{mathtools}
\newcommand{\gc}[1]{\textcolor{gray}{#1}}

\usepackage{graphicx}
\usepackage{subcaption} 
\title{Flow Matching with Arbitrary Auxiliary Paths}

\author{%
  Xin Peng$^{1}$ \quad
  Ang Gao$^{1}$\thanks{Corresponding author} \\
  $^{1}$School of Physical Science and Technology, \\
  Beijing University of Posts and Telecommunications, Beijing, China \\
  \texttt{anggao@bupt.edu.cn}
}

\begin{document}

\maketitle

\begin{abstract}
We introduce a new generative modeling framework, \textbf{Flow Matching with Arbitrary Auxiliary Paths (AuxPath-FM)}, which generalizes conditional flow matching by incorporating an auxiliary variable drawn from an arbitrary distribution into the probability path. 
Unlike prior methods that restrict auxiliary components to Gaussian noise, AuxPath-FM allows the variable $\eta$ to follow any distribution, producing trajectories of the form $X_t = a(t)X_1 + b(t)X_0 + c(t)\eta$. 
We theoretically demonstrate that this construction preserves the continuity equation and maintains a training objective consistent with the marginal formulation. 
This flexibility enables the design of diverse probability paths using various priors, including Gaussian, Uniform, Laplace, and discrete Rademacher distributions, each offering unique geometric properties for generative flows. 
Furthermore, our framework allows for specialized tasks such as label-guided generation by encoding structured semantic information into the auxiliary distribution. 
Overall, AuxPath-FM provides a principled and general foundation for probability path design, offering both theoretical generality and practical flexibility for diverse generative modeling tasks.
\end{abstract}

\section{Introduction}

Continuous-time generative models, such as score-based diffusion models~\cite{song2021score,ho2020ddpm} and flow matching frameworks~\cite{lipman2023flow,liu2023rectified,tong2023conditional}, have emerged as powerful tools for modeling complex data distributions by learning dynamics that transport a simple base distribution to a target distribution. In conditional generative modeling, incorporating label or side information is crucial for controllable sample generation. However, most existing approaches introduce such information only through conditioning neural networks~\cite{ho2022classifier,dhariwal2021diffusion}, rather than modifying the underlying probability paths. While recent advances have expanded flow matching to multi-scale representations~\cite{kumar2026laplacian}, physical systems with hard constraints~\cite{chen2025physics}, and complex temporal forecasting~\cite{spectflow2025}, these models still primarily rely on network-side conditioning. Furthermore, recent theoretical analysis suggests that flow matching is naturally adaptive to the intrinsic manifold structure of data~\cite{kumar2026manifold}, yet the influence of the path's stochastic component remains under-explored in the context of trajectory-level control.

Meanwhile, prior path-based formulations, including stochastic interpolants~\cite{albergo2023stochastic} and Schrödinger Bridge methods~\cite{de2021diffusion,chen2023schrodinger,pooladian2023multimarginal}, typically rely on Gaussian noise or Brownian dynamics for stochastic interpolation, which limits the flexibility of designing probability paths for structured guidance. Although recent research has introduced discrete state-space flows via edit operations~\cite{havasi2025edit,luedke2026editpp} and geometric flows on statistical manifolds~\cite{davis2024fisher}, these methods often remain specialized to particular data types. Modern alignment techniques, such as policy gradient methods for flow matching~\cite{liu2025flowgrpo}, further highlight the need for more expressive generative trajectories that can be steered via diverse signals. 

In this work, we introduce \textbf{Flow Matching with Arbitrary Auxiliary Paths (AuxPath-FM)}, a general framework that extends flow matching by incorporating an auxiliary random variable into the probability path. Specifically, we construct trajectories of the form
\begin{equation}
X_t = a(t)X_1 + b(t)X_0 + c(t) \eta,
\end{equation}
where $X_0 \sim p_0$, $X_1 \sim p_1$, and $\eta$ is an auxiliary random variable drawn from an arbitrary distribution. Unlike prior interpolation schemes that restrict auxiliary terms to Gaussian noise~\cite{albergo2023stochastic,de2021diffusion}, our formulation allows flexible auxiliary designs such as condition-dependent sources~\cite{csfm2026} or structured transition kernels~\cite{shaul2025transition}, enabling the auxiliary variable to directly encode structured signals in the generative trajectory.

\begin{figure}[t]
    \centering
    \includegraphics[width=0.80\textwidth]{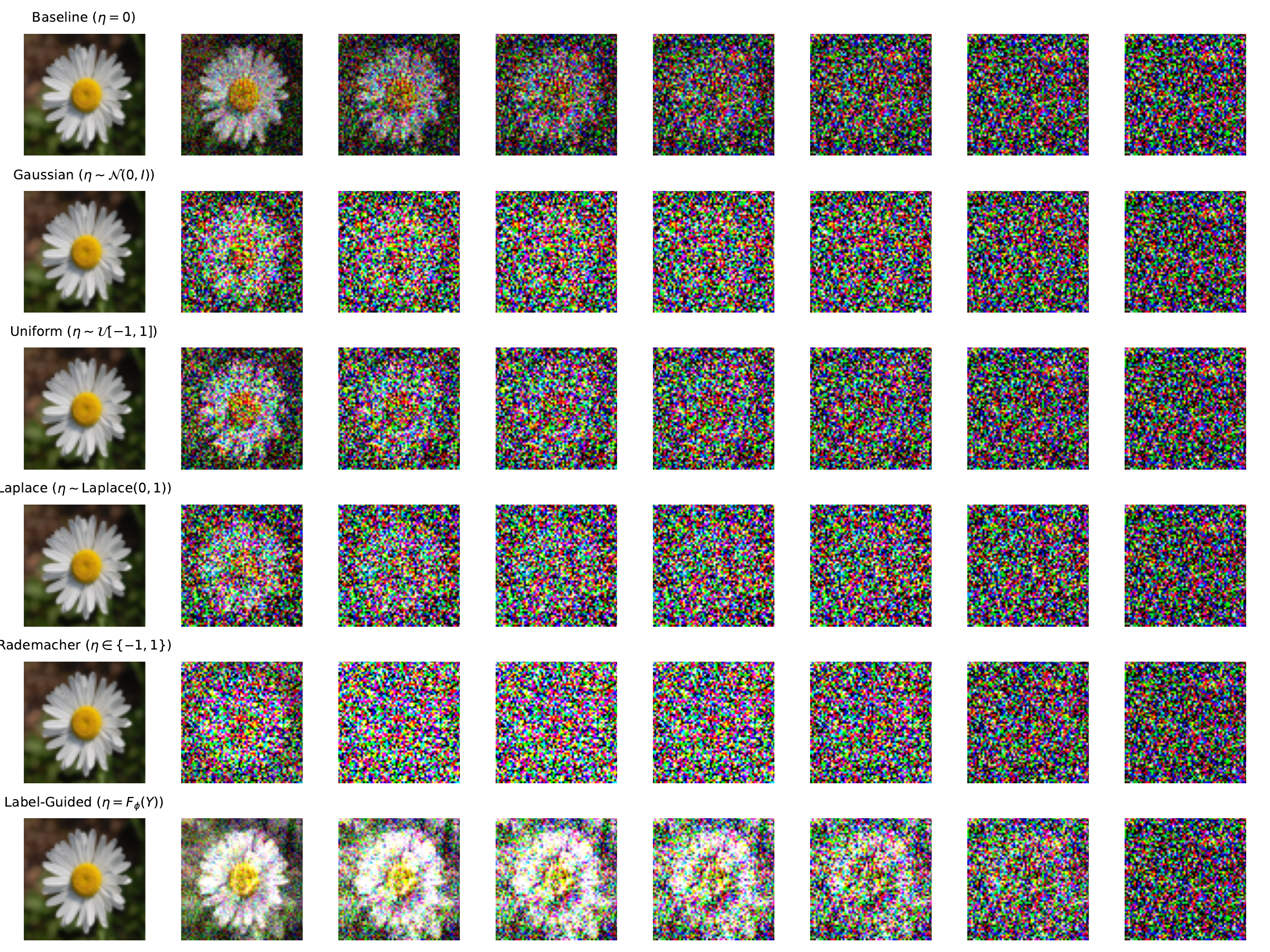} 
    \caption{\textbf{Generative trajectories of AuxPath-FM with diverse auxiliary distributions.} 
    The figure illustrates the interpolation paths $X_t = a(t)X_1 + b(t)X_0 + c(t)\eta$ between two data samples $X_0$ and $X_1$. 
    By varying the distribution of the auxiliary variable $\eta$, we observe unique geometric properties: 
    (a) Baseline ($\eta=0$), (b) Gaussian $\eta \sim \mathcal{N}(0, I)$, (c) Uniform $\eta \sim \mathcal{U}([-1, 1]^d)$, 
    (d) Laplace $\eta \sim \text{Laplace}(0, 1)$, (e) Rademacher $\eta \in \{-1, 1\}^d$, 
    and (f) Label-Guided $\eta = F_\phi(Y)$. 
    Note how discrete and heavy-tailed priors impart distinct structural biases to the flow.}
    \label{fig:auxpath_interpolation}
\end{figure}

We show that this construction preserves the continuity equation and yields a flow matching objective consistent with the marginal formulation~\cite{lipman2023flow,albergo2023stochastic}. A particularly important instance arises when $\eta$ encodes semantic information such as class labels, enabling label-guided generation through the probability path itself. In contrast to conventional conditioning approaches and classifier-free guidance~\cite{ho2022classifier,dhariwal2021diffusion}, which operate through network-level modulation, our method embeds guidance directly into the trajectory design, leading to principled control of the generative dynamics. Overall, AuxPath-FM generalizes existing path-based frameworks such as stochastic interpolants and Schrödinger Bridge methods~\cite{albergo2023stochastic,de2021diffusion,chen2023schrodinger}, while retaining the standard flow matching training objective, thereby providing a unified and flexible perspective for structured continuous-time generative modeling.

\section{Related Work}

\paragraph{Diffusion and Flow Matching.}
Continuous-time generative models have achieved remarkable success in modeling complex data distributions~\cite{song2021score,ho2020ddpm,kingma2021variational}. Score-based diffusion models learn the gradient of the log-density and simulate reverse-time stochastic differential equations to transform noise into data samples~\cite{song2021score}. More recently, flow matching (FM) provides an alternative framework that directly learns the velocity field of a deterministic probability flow connecting a base distribution and a data distribution~\cite{lipman2023flow,liu2023rectified}. Instead of simulating stochastic processes, FM trains a neural network to match the velocity of a predefined probability path, enabling efficient training and flexible trajectory design~\cite{lipman2023flow,tong2023conditional}. Recent theoretical advances have further established that flow matching is naturally adaptive to manifold structures~\cite{kumar2026manifold}, and generalizations such as Transition Matching~\cite{shaul2025transition} have unified flow models with discrete-time Markovian transitions to improve sampling efficiency and quality.

\paragraph{Schrödinger Bridges and Stochastic Interpolants.}
Several works study generative modeling through stochastic interpolation between distributions~\cite{de2021diffusion,chen2023schrodinger,pooladian2023multimarginal,frogner2019learning}. Schrödinger Bridge methods formulate generative modeling as an entropy-regularized optimal transport problem, producing stochastic trajectories between endpoints~\cite{de2021diffusion,chen2023schrodinger,benamou2015iterative}. More recently, stochastic interpolants provide a unifying framework connecting diffusion and flow-based methods~\cite{albergo2023stochastic,huang2021stochastic}. In many such formulations, interpolation paths typically introduce auxiliary noise variables that follow Gaussian distributions~\cite{albergo2023stochastic,de2021diffusion}. However, recent extensions have begun to explore flow matching on non-Euclidean geometries, such as Riemannian manifolds~\cite{chen2024geometries} and periodic lattices for material discovery~\cite{miller2024flowmm}, where the geometry of the space dictates the design of the interpolation path beyond simple Euclidean Gaussianity.

\paragraph{Alternative Forward Transformations and Non-Gaussian Paths.}
Some works explore relaxing the standard Gaussian noise assumption. For example, Cold Diffusion replaces Gaussian noise with deterministic image transformations~\cite{hoogeboom2021cold}. Other works investigate generalizing diffusion beyond Gaussian perturbations via structured noise models~\cite{song2021score,kingma2021variational}. Notably, recent developments in discrete state-spaces have introduced Dirichlet Flow Matching~\cite{stark2024dirichlet} and Discrete Flow Matching~\cite{gat2024discrete,campbell2024generative}, which utilize probability paths on simplices or categorical spaces to handle discrete sequential data. These studies suggest that generative modeling can be viewed as learning to reverse a sequence of transformations. Our work further generalizes this perspective by allowing the auxiliary path components to follow any arbitrary distribution, providing a unified view of both continuous and discrete-inspired priors.

\paragraph{Conditional Generative Modeling and Guidance.}
Conditional generative models incorporate side information such as class labels or attributes to enable controllable synthesis~\cite{mirza2014cgan,dhariwal2021diffusion,ho2022classifier}. A common strategy conditions the neural network on label embeddings while keeping the probability path unchanged~\cite{mirza2014cgan,dhariwal2021diffusion}. Classifier-free guidance further improves results by combining conditional and unconditional models during sampling~\cite{ho2022classifier}. Recent methods extend this to complex scientific domains, such as sequence-conditioned protein structure generation in FoldFlow-2~\cite{huguet2024foldflow2} and contrastive flow matching to enhance condition separation~\cite{stoica2025contrastive}. While some approaches explore augmenting flows with auxiliary variables to reduce training variance~\cite{lee2024augmented}, most current techniques primarily introduce conditioning through network inputs rather than modifying the geometric structure of the probability paths themselves.

\paragraph{Our Contribution.}
In contrast to existing approaches, we introduce \textbf{Flow Matching with Arbitrary Auxiliary Paths (AuxPath-FM)}, which generalizes probability path design in flow matching by allowing the auxiliary component of the path to follow an arbitrary distribution.
This perspective unifies several existing interpolation schemes as special cases while enabling auxiliary variables, such as label-induced distributions, to directly shape the generative trajectory.
As a result, our framework provides a principled mechanism for controllable generation while preserving the efficiency and simplicity of the flow matching objective.

\section{Background}

\paragraph{Flow Matching.}
Flow Matching provides a simple and scalable framework for training continuous-time generative models~\cite{lipman2023flow,liu2023rectified}.
Instead of learning the score function of a stochastic diffusion process~\cite{song2021score,ho2020ddpm}, it directly parameterizes a velocity field that transports samples from a base distribution $p_0(x)$ to a target data distribution $p_1(x)$ along a predefined probability path $p_t(x)$ for $t \in [0,1]$.

Let $X_0 \sim p_0$ and $X_1 \sim p_1$. A common choice of interpolation is the linear path $X_t = (1-t)X_0 + tX_1$, whose time derivative is $\dot{X}_t = X_1 - X_0$.
The objective of Flow Matching is to learn a velocity field $v_\theta(x,t)$ that matches this conditional velocity:
\begin{equation}
\mathcal{L}_{FM} =
\mathbb{E}_{t,X_0,X_1}
\left[
\| v_\theta(X_t,t) - (X_1 - X_0) \|^2
\right].
\end{equation}

After training, generation is performed by solving the ordinary differential equation (ODE)
\begin{equation}
\frac{dX_t}{dt} = v_\theta(X_t,t),
\end{equation}
which transports samples from $p_0$ to $p_1$.

\paragraph{Conditional Flow Matching.}
Flow Matching can be extended to conditional generative modeling by introducing an auxiliary variable $Y$. (e.g., class labels or side information)~\cite{tong2023conditional}.
The velocity field is conditioned as $v_\theta(x,t,Y)$, while the interpolation path typically remains $X_t = (1-t)X_0 + tX_1$, where $X_0 \sim p_0$ and $X_1 \sim p_1(\cdot|Y)$.

The resulting objective becomes
\begin{equation}
\mathcal{L}_{CFM} =
\mathbb{E}_{t,X_0,X_1,Y}
\left[
\| v_\theta(X_t,t,Y) - (X_1 - X_0) \|^2
\right].
\end{equation}

Although conditioning is incorporated through the network, the probability path itself is independent of $Y$.
Thus, side information influences the learned dynamics only implicitly via the parameterization of $v_\theta$, rather than through the path design.

\paragraph{Stochastic Interpolants and Schrödinger Bridges.}
Recent works extend deterministic interpolation by introducing stochasticity along the path~\cite{albergo2023stochastic,de2021diffusion,chen2023schrodinger}.
A general stochastic interpolant can be written as
\begin{equation}
X_t = (1-t)X_0 + tX_1 + \sigma(t) Z,
\end{equation}
where $X_0 \sim p_0$, $X_1 \sim p_1$, and $Z \sim \mathcal{N}(0, I)$.

Such formulations arise in stochastic interpolant models, diffusion models, and Schrödinger bridge methods~\cite{song2021score,de2021diffusion,chen2023schrodinger}.
The additive Gaussian term $\sigma(t) Z$ injects noise, enabling more flexible transport between distributions while preserving tractable training objectives.

However, existing approaches typically restrict the auxiliary noise to be Gaussian.
This constraint limits the expressiveness of the interpolation path and prevents incorporating more structured auxiliary signals that could better guide the generative dynamics.

\section{Method}

In this section, we introduce \textbf{Flow Matching with Arbitrary Auxiliary Paths (AuxPath-FM)}, a general framework that extends conditional flow matching by incorporating auxiliary variables drawn from arbitrary distributions into probability paths. Unlike conventional approaches that inject auxiliary signals only through network inputs, our formulation allows these signals to directly shape the generative trajectory.

\section{Flow Matching with Arbitrary Auxiliary Paths}

In this section, we introduce \textbf{Flow Matching with Arbitrary Auxiliary Paths (AuxPath-FM)}, a general framework that extends conditional flow matching by incorporating auxiliary variables drawn from arbitrary distributions into probability paths. Unlike conventional approaches that inject auxiliary signals only through network inputs, our formulation allows these signals to directly shape the generative trajectory.

\subsection{Auxiliary Probability Paths}
\label{sub:A-FM}

Standard flow matching constructs paths that interpolate between base samples $X_0 \sim p_0$ and data samples $X_1 \sim p_1$. To enhance path flexibility, we introduce an auxiliary random variable $\eta \sim p_{\eta}$ drawn from a general distribution, and define
\begin{equation}
    X_t = a(t)X_1 + b(t)X_0 + c(t)\eta,
\end{equation}
where $a(t), b(t), c(t)$ are time-dependent scalar functions.

To ensure correct boundary behavior, the coefficients satisfy
$a(0)=0, b(0)=1, c(0)=0$ and $a(1)=1, b(1)=0, c(1)=0$, such that the path transports samples from $p_0$ to $p_1$. This formulation enables the trajectory to incorporate structured signals beyond standard noise.

\paragraph{Conditional probability paths.}
Conditioned on $(X_0, \eta)$, the trajectory induces a conditional probability path $p_t(x \mid x_0, \eta)$ together with a conditional vector field $u_t(x \mid x_0, \eta)$ that governs the dynamics of $X_t$. This provides a family of simple paths indexed by $(x_0, \eta)$.

\paragraph{Marginal probability path.}
The overall probability path is obtained by marginalizing over the joint distribution of $(X_0, \eta)$:
\begin{equation}
    p_t(x) = \iint p_t(x \mid x_0, \eta)\, p_0(x_0)\, p_{\eta}(\eta)\, dx_0\, d\eta.
\end{equation}

\paragraph{Marginal vector field.}
Analogously, we define the marginal velocity field by aggregating conditional vector fields:
\begin{equation}
    u_t(x) =
    \frac{
    \iint p_t(x \mid x_0, \eta)\, u_t(x \mid x_0, \eta)\, p_0(x_0)\, p_{\eta}(\eta)\, dx_0\, d\eta
    }{p_t(x)}.
\end{equation}

\begin{theorem}[Consistency of Marginal Flow]
\label{thm:marginal_flow_aux}
Let $u_t(x \mid x_0, \eta)$ be conditional vector fields that generate conditional probability paths $p_t(x \mid x_0, \eta)$ satisfying the continuity equation. Then the marginal vector field $u_t(x)$ defined above generates the marginal probability path $p_t(x)$, i.e.,
\begin{equation}
    \partial_t p_t(x) = - \nabla_x \cdot \big( p_t(x)\, u_t(x) \big).
\end{equation}
\end{theorem}
The detailed proof follows from standard properties of conditional continuity equations and marginalization of probability flows. We defer the complete derivation to Appendix~\ref{sec:AuxPath_model}.

\paragraph{Implication.}
Theorem~\ref{thm:marginal_flow_aux} shows that the global flow can be decomposed into a mixture of simpler conditional flows. This provides a principled justification for designing complex probability paths via auxiliary variables.

\paragraph{Choice of auxiliary distribution.}
Importantly, the auxiliary distribution $p_{\eta}$ is not restricted to a specific form. For instance, $\eta$ can be drawn from Gaussian distributions, bounded distributions, heavy-tailed distributions, or even discrete distributions. This flexibility allows the probability path to incorporate diverse structural priors.
\input{code}

Different choices of $p_{\eta}$ correspond to different mixtures of conditional flows, thereby inducing distinct geometric properties of the overall trajectory. We empirically investigate the impact of different choices of $p_{\eta}$ in Sec.~\ref{sub:toy}.

\subsection{Velocity Field and Training Objective}

Let $v_\theta(x,t)$ denote the neural network parameterizing the velocity field. 
Under the auxiliary probability path defined in Sec.~\ref{sub:A-FM}, the trajectory admits the time derivative
\begin{equation}
    \dot{X}_t = \dot{a}(t)X_1 + \dot{b}(t)X_0 + \dot{c}(t)\eta.
\end{equation}

\paragraph{From marginal to conditional objectives.}
In principle, one may define a flow matching objective on the marginal probability path $p_t(x)$:
\begin{equation}
    \mathcal{L}_{\mathrm{FM}} =
    \mathbb{E}_{t, x \sim p_t}
    \left[
    \| v_\theta(x,t) - u_t(x) \|^2
    \right],
\end{equation}
where $u_t(x)$ is the marginal vector field defined in Sec.~\ref{sub:A-FM}.

However, both $p_t(x)$ and $u_t(x)$ involve intractable integrals over $(X_0, \eta)$, making it impractical to directly compute unbiased estimates of this objective.

\paragraph{AuxPath Conditional Flow Matching.}
To address this issue, we instead optimize a conditional objective defined on the tractable conditional paths:
\begin{equation}
    \mathcal{L}_{\mathrm{AuxPath}} =
    \mathbb{E}_{t, X_0, X_1, \eta}
    \left[
    \| v_\theta(X_t,t) - \dot{X}_t \|^2
    \right],
\end{equation}
where $X_t = a(t)X_1 + b(t)X_0 + c(t)\eta$.
\input{code1}

This objective corresponds to matching the conditional vector field $u_t(x \mid X_0, \eta)$, which can be computed analytically for each sample. As a result, $\mathcal{L}_{\mathrm{AuxPath}}$ admits efficient and unbiased Monte Carlo estimation.

\paragraph{Equivalence to marginal flow matching.}
Although $\mathcal{L}_{\mathrm{AuxPath}}$ is defined on conditional paths, it is in fact equivalent to the marginal flow matching objective in expectation.

\begin{theorem}[Loss Equivalence]
\label{thm:loss_equivalence_main}
Assume $p_t(x) > 0$ for all $x$ and $t \in [0,1]$. Then, up to a constant independent of $\theta$,
\begin{equation}
    \mathcal{L}_{\mathrm{FM}}(\theta)
    =
    \mathcal{L}_{\mathrm{AuxPath}}(\theta)
    + C,
\end{equation}
and consequently,
\begin{equation}
    \nabla_\theta \mathcal{L}_{\mathrm{FM}}(\theta)
    =
    \nabla_\theta \mathcal{L}_{\mathrm{AuxPath}}(\theta).
\end{equation}
\end{theorem}

\paragraph{Implication.}
Theorem~\ref{thm:loss_equivalence_main} shows that optimizing the conditional objective yields the same optimal solution as the intractable marginal objective. Therefore, we can learn a velocity field that generates the marginal probability path $p_t(x)$—and in particular approximates the data distribution at $t=1$—without explicitly computing either $p_t(x)$ or $u_t(x)$.

\subsection{Semantic Auxiliary Variables for Conditional Generation}

A natural extension of AuxPath-FM is to consider auxiliary variables that encode semantic information, enabling conditional generation. In this setting, we define $\eta = F_\phi(Y)$, where $Y$ denotes a class label and $F_\phi$ is a lightweight neural network.

To provide meaningful guidance, we employ \textbf{prototype learning}, training $F_\phi$ to approximate the class-wise centroid:
\begin{equation}
F_\phi(Y) \approx \mathbb{E}[X_1 \mid Y].
\end{equation}
This is implemented by minimizing
\begin{equation}
\mathcal{L}_{F} =
\mathbb{E}_{X_1,Y}
\left[
\left\|
F_\phi(Y) - X_1
\right\|^2
\right].
\end{equation}

This encourages $F_\phi(Y)$ to capture the average structure of each class, effectively guiding trajectories toward class-specific regions of the data distribution.

Within the AuxPath-FM framework, this corresponds to instantiating the auxiliary variable $\eta$ in the probability path with a label-dependent representation. The resulting trajectories are therefore biased toward semantically consistent regions throughout the flow, rather than only at the endpoint.

Notably, since $F_\phi$ is significantly smaller than the main velocity network, this additional objective introduces negligible computational overhead. More generally, the auxiliary variable $\eta$ can encode diverse forms of structured information beyond labels, making the framework broadly applicable to conditional and guided generative modeling.

\subsection{Trajectory-Level Classifier-Free Guidance}

AuxPath-FM admits a trajectory-level formulation of classifier-free guidance (CFG)~\cite{ho2022classifier}. 
Specifically, the conditional and unconditional velocity fields take the form
\begin{equation}
v(x,t,y) = v_\theta(x,t) + \dot{c}(t) F_\phi(y), \quad
v(x,t,\emptyset) = v_\theta(x,t) + \dot{c}(t) F_\phi(\emptyset).
\end{equation}
Applying standard CFG yields
\begin{equation}
v_{\mathrm{cfg}}(x,t)
= v_\theta(x,t) + \dot{c}(t)\Big[ F_\phi(\emptyset) + w\big(F_\phi(y)-F_\phi(\emptyset)\big) \Big].
\end{equation}

This formulation decouples guidance from the backbone model: unlike conventional CFG, which requires two evaluations of $v_\theta$, our approach reuses a single evaluation and applies guidance through the lightweight module $F_\phi$. 
Full derivation and discussion are provided in Appendix~\ref{sec:auxpath_cfg}.

\subsection{Toy Experiment: Visualizing Trajectory-Level Guidance}
\label{subsec:toy_cfg}
\begin{wrapfigure}{r}{0.45\linewidth}
\vspace{-1.0em}
\centering
\includegraphics[width=\linewidth]{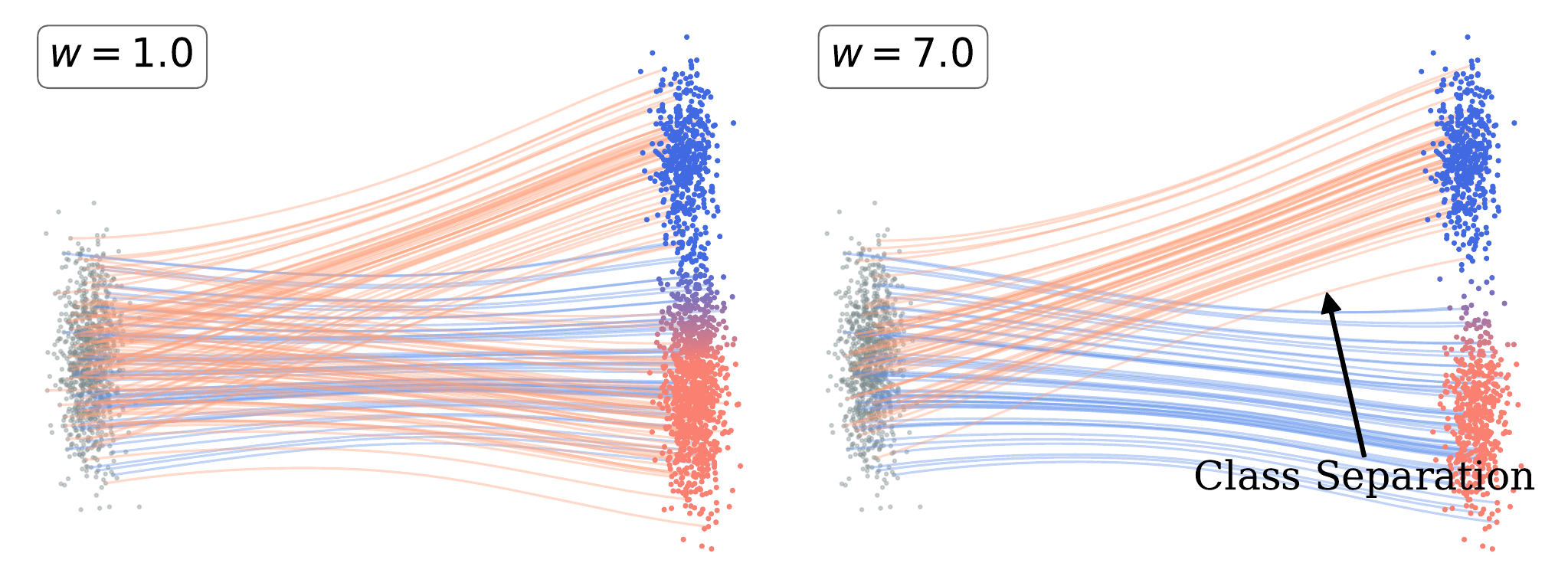}
\vspace{-0.8em}
\caption{\textbf{Trajectory-level guidance.} 
Comparison between $w=1.0$ and $w=7.0$. Higher guidance leads to clear trajectory separation.}
\label{fig:auxpath_cfg_toy}
\vspace{-1.0em}
\end{wrapfigure}

We illustrate trajectory-level CFG in a 2D toy setting with a bimodal Gaussian dataset (two clusters on a ring). 
The flow model is trained with AuxPath-FM, and guidance is applied at sampling using Algorithm~\ref{alg:auxpath_cond_cfg}.

As shown in Fig.~\ref{fig:auxpath_cfg_toy}, increasing the guidance scale to $w=7$ leads to clear separation of trajectories across modes. 
This confirms that CFG can be effectively applied at the trajectory level, where the auxiliary term acts as a global drift without requiring repeated evaluations of the backbone model.

\subsection{Toy Experiment: Ring-64 Dataset}
\label{sub:toy}

We evaluate AuxPath-FM on a ring-64 dataset with 64 modes on a unit circle, requiring finer multi-modal modeling than ring-8 and providing a stricter test for auxiliary design.

All methods share the same velocity network $v_\theta(X_t,t,Y)$ and differ only in the auxiliary variable $\eta$. CFM uses $X_t=(1-t)X_0+tX_1$, while AuxPath-FM adds $t(1-t)\eta$, with $X_0\sim\mathcal{N}(0,I)$ and $X_1\sim p_1$. We consider Gaussian, Uniform, Laplace, Rademacher, and label-guided $\eta = F(Y)$, where $F_\phi(Y)\approx\mathbb{E}[X\mid Y]$.

We report mode accuracy (Acc), i.e., correct assignment to the nearest mode, and distance error (Error), i.e., average distance to mode centers. As shown in Fig.~\ref{fig:ring64_auxpath}, all variants preserve correct transport, while Table~\ref{tab:ring64_results} shows that label-guided $\eta$ achieves the best performance, with higher Acc and lower Error, indicating improved controllable multi-modal generation.

\begin{figure}[h]
\centering
\scalebox{0.98}{
\begin{minipage}{0.48\columnwidth}
    \centering
    \includegraphics[width=\linewidth]{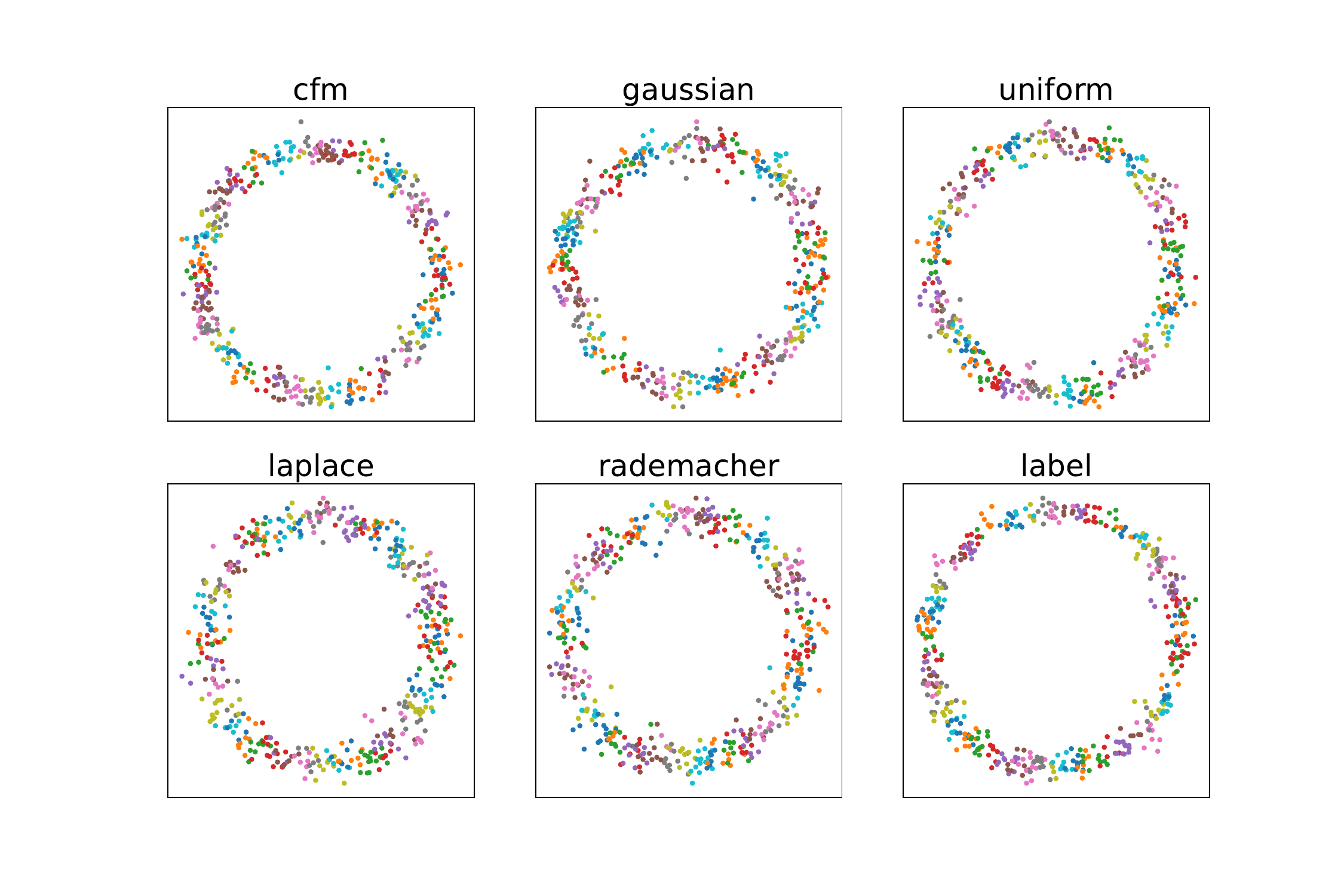}
    \caption{
    Generated trajectories on the ring-64 dataset using different auxiliary distributions \(\eta\). }
    \label{fig:ring64_auxpath}
\end{minipage}
\hfill
\begin{minipage}{0.48\columnwidth}
    \centering
    \setlength{\tabcolsep}{2pt}
    \captionof{table}{Evaluation metrics on the ring-64 dataset. Higher accuracy and lower error indicate better performance.}
    \label{tab:ring64_results}
    \begin{tabular}{lcc}
    \hline
    Auxiliary Type & Acc (\%) $\uparrow$ & Error $\downarrow$ \\
    \hline
    CFM (baseline)       & 57.20 & 0.3398 \\
    Gaussian             & 45.60 & 0.3956 \\
    Uniform              & 49.80 & 0.3586 \\
    Laplace              & 47.20 & 0.4228 \\
    Rademacher           & 51.00 & 0.3911 \\
    Label-guided (Ours)  & \textbf{60.60} & \textbf{0.2955} \\
    \hline
    \end{tabular}
\end{minipage}}
\end{figure}

\section{Experiments}

We evaluate AuxPath-FM on multiple datasets to validate its effectiveness in conditional generation and the flexibility of auxiliary path design. All experiments are conducted on a single NVIDIA RTX 4090 GPU.

\subsection{Experimental Setup}
\begin{wrapfigure}{r}{0.5\linewidth}
\vspace{-1.0em}
\centering
\includegraphics[width=\linewidth]{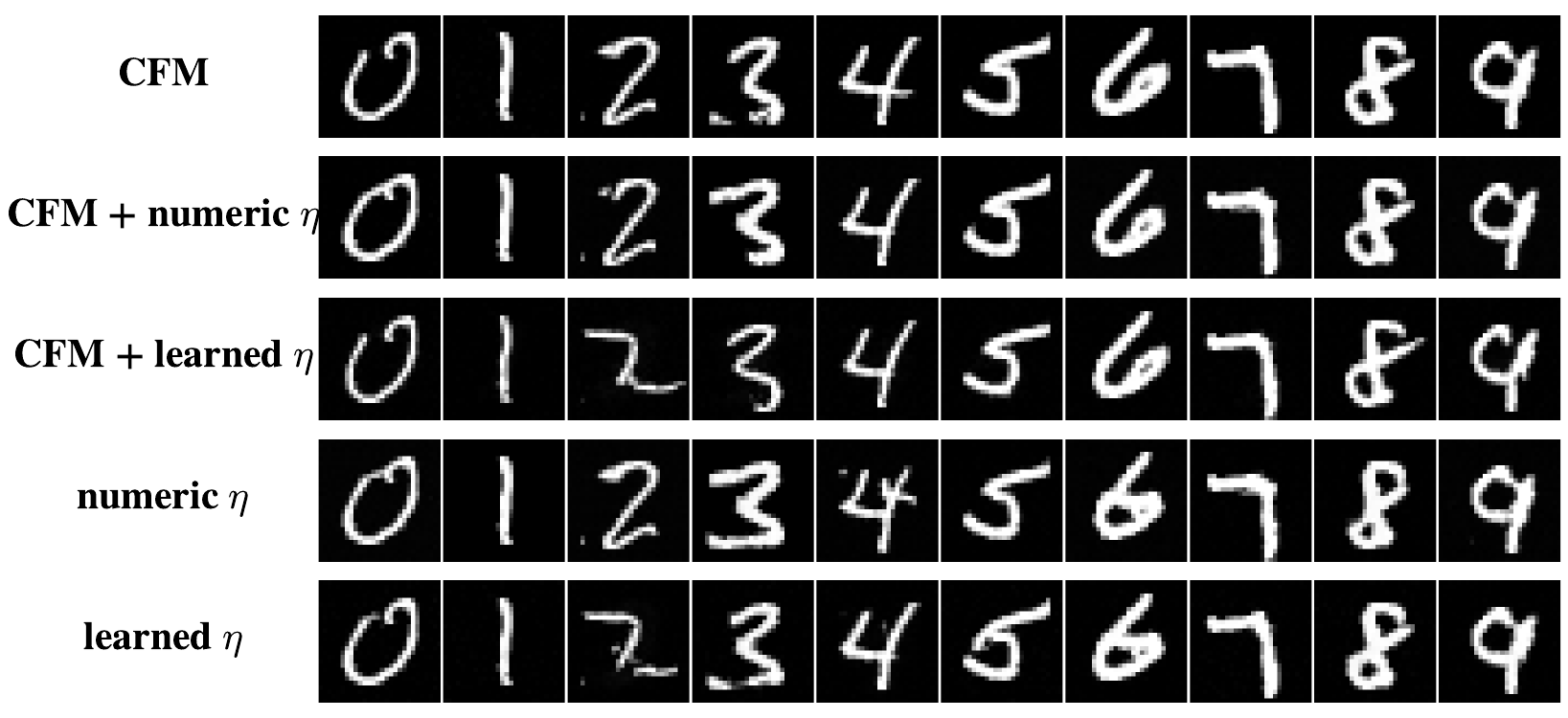}
\vspace{-0.8em}
\caption{\textbf{Qualitative comparison of conditional generation on MNIST.} 
We compare generated samples across different model variants: standard CFM (baseline), path-guided models with fixed numeric auxiliary variables ($\eta = \text{numeric}$), and AuxPath-FM using learned label prototypes ($\eta = F_\phi(Y)$).}
\label{fig:mnist_comparison}
\vspace{-1.0em}
\end{wrapfigure}
\paragraph{Datasets.}
We consider three datasets:MNIST \cite{lecun1998mnist}, CIFAR-10 \cite{krizhevsky2009cifar}, and ImageNet-1k \cite{imagenet2009} are used to evaluate conditional generation accuracy, while ImageNet-1k is used to test scalability and distribution-agnostic auxiliary paths.

\paragraph{Evaluation Metrics.}
We report the following metrics:
(i) Classification Accuracy (Acc), computed using pretrained classifiers to evaluate conditional generation quality. 
For MNIST, we adopt a CNN classifier~\cite{cnn_mnist_ref} achieving 99.1\% accuracy; for CIFAR-10, we use a ResNet-56 classifier~\cite{he2016resnet} with 92\% accuracy. 
These classifiers are used to assess whether generated samples match the intended labels. 
(ii) \textbf{Fr\'{e}chet Inception Distance (FID)} \cite{heusel2017fid}, \textbf{Spatial FID (sFID)} \cite{parmar2022sfid}, and \textbf{Inception Score (IS)} \cite{salimans2016is} for generative quality evaluation. 
All metrics are computed using 50k generated samples following standard evaluation protocols.

\subsection{Conditional Generation Performance (MNIST and CIFAR-10)}
\label{sub:1}

\paragraph{Experimental setup.}
For low-resolution datasets, we adopt different backbones for each setting: a standard \textbf{U-Net}~\cite{ronneberger2015u,ho2020ddpm} for CIFAR-10, and \textbf{LightningDiT}~\cite{yao2025reconstruction} for MNIST. 
For a fair comparison, all methods (including CFM and its variants) are trained for 400K steps under identical settings. 
Our conditional AuxPath-FM follows the two-stage training procedure in Algorithm~\ref{alg:auxpath_cond_train} and performs sampling using Algorithm~\ref{alg:auxpath_cond_sample}. 
Importantly, the learned velocity model $v_\theta$ is trained to match the baseline flow field without the auxiliary component, while the auxiliary variable $\eta = F_\phi(Y)$ is incorporated only through the probability path during training and reintroduced at inference time to construct the full dynamics via $v_\theta(x,t) + \dot{c}(t)\eta$. 
Qualitative results on MNIST are shown in Fig.~\ref{fig:mnist_comparison}, comparing standard CFM (pure conditional guidance), auxiliary path guidance alone (via $\eta$), and their combination (CFM + $\eta$), where $\eta$ is instantiated either as numeric values (fixed numeric filling) or learned semantic prototypes.

\begin{table}[t]
\centering
\caption{Results on MNIST and CIFAR-10. Acc is reported using pretrained classifiers.}
\begin{tabular}{lcccccccc}
\toprule
& \multicolumn{4}{c}{MNIST} & \multicolumn{4}{c}{CIFAR-10} \\
\cmidrule(lr){2-5} \cmidrule(lr){6-9}
Method 
& Acc(\%) $\uparrow$ & FID $\downarrow$ & sFID $\downarrow$ & IS $\uparrow$
& Acc(\%) $\uparrow$ & FID $\downarrow$ & sFID $\downarrow$ & IS $\uparrow$ \\
\midrule
CFM & 95.5& \cellcolor{gray!15}\textbf{5.90} & \cellcolor{gray!15}\textbf{0.006}& 2.099 &82.24 &4.71 & 0.0032 & 9.18 \\
CFM + numeric $\eta$ & \cellcolor{gray!15}\textbf{98.1}&7.33 & 0.007 & 2.096 & 83.61 & 4.54 & 0.0030 & 9.19 \\
CFM + learned $\eta$ & 97.1 & 7.57 & 0.008 & 2.112 & \cellcolor{gray!15}\textbf{87.28} & \cellcolor{gray!15}\textbf{3.68} & \cellcolor{gray!15}\textbf{0.0017} & \cellcolor{gray!15}\textbf{9.20} \\
numeric $\eta$ & 94.3 &10.04  & 0.009 &2.097 & 79.18 & 5.73 & 0.0038 & 8.89 \\
learned $\eta$ & 90.6 & 13.82  & 0.015 &\cellcolor{gray!15}\textbf{2.128}&  77.03 & 5.90 & 0.0043 & 8.93 \\
\bottomrule
\end{tabular}
\label{tab:mnist_cifar_results}
\end{table}

\paragraph{Results.}
As shown in Fig.~\ref{fig:mnist_comparison} and Table~\ref{tab:mnist_cifar_results}, both standard conditional guidance (CFM) and auxiliary path guidance (via $\eta$) can independently achieve conditional generation, while auxiliary guidance alone (numeric $\eta$) yields slightly lower accuracy. 
When combined (CFM + $\eta$), the two mechanisms provide complementary effects and consistently improve semantic accuracy, demonstrating that auxiliary path guidance enhances standard conditional generation. 
Among different instantiations of $\eta$, learned semantic prototypes further improve consistency and fidelity compared to numeric ones. 
On MNIST, this results in the highest classification accuracy with competitive FID and sFID, and on CIFAR-10, consistent improvements are observed across all metrics, especially in FID and sFID.

\subsection{Conditional Flow Matching with Different Auxiliary Distributions}
\begin{wraptable}{r}{0.45\linewidth}
\vspace{-1.0em}
\centering
\small
\caption{Results on CIFAR-10 and ImageNet-1k.}
\label{tab:all_results}

\setlength{\tabcolsep}{2.5pt}
\begin{tabular}{lccc c}
\toprule
& \multicolumn{3}{c}{CIFAR-10} & ImageNet-1k \\
\cmidrule(lr){2-4} \cmidrule(lr){5-5}
Method 
& FID $\downarrow$ & sFID $\downarrow$ & IS $\uparrow$
& FID $\downarrow$ \\
\midrule
\gc{CFM}                & \gc{4.71}  & \gc{0.0032} & \gc{9.18} & \gc{21.13} \\
Gaussian           & \cellcolor{gray!15}\textbf{4.38}  & 0.0029 & 9.21 & 24.19 \\
Uniform            & 4.40   & 0.0028    & 9.17  & 23.48 \\
Laplace            & 4.49  & 0.0031 & 9.19 & 24.71 \\
Rademacher         & 4.46  & 0.0031    & 9.24  & 24.23 \\
Label-guided & 4.59    & \cellcolor{gray!15}\textbf{0.0024}   & \cellcolor{gray!15}\textbf{9.32}  & \cellcolor{gray!15}\textbf{23.45} \\
\bottomrule
\end{tabular}

\vspace{-0.8em}
\end{wraptable}
\paragraph{Setup.}
To assess scalability to high-resolution manifolds ($256\times256$), we adopt a \textbf{Diffusion Transformer (DiT-B/2)}~\cite{peebles2023dit}. 
Following prior work, the flow model is initialized from a pre-trained conditional flow matching (CFM) checkpoint~\cite{dao2023flow} (originally trained for $\sim$10M steps) and further optimized for only 300K steps, corresponding to roughly 3\% of the original training budget. 
All methods share the same architecture and training protocol, and differ only in the choice of auxiliary variable $\eta$. 
AuxPath-FM is trained using Algorithm~\ref{alg:auxpath_train} and sampled using Algorithm~\ref{alg:auxpath_sample}. 
Importantly, $\eta$ is used only as an auxiliary component during training, while the learned velocity field $v_\theta$ preserves the underlying CFM dynamics and is directly used at inference time.

\paragraph{Results.}
Table~\ref{tab:all_results} reports FID across datasets. 
We observe that under various choices of auxiliary distributions (Gaussian, Uniform, Laplace, and Rademacher), the model consistently maintains stable generative performance, indicating that incorporating $\eta$ does not disrupt the conditional flow matching process. 
These results show that different auxiliary distributions can be flexibly integrated into CFM while retaining its ability to generate high-quality samples, demonstrating the robustness of the proposed framework.

\subsection{Effect of CFG Scale}
\label{subsec:cfg_effect}
\begin{wraptable}{r}{0.48\linewidth}
\vspace{-1.0em}
\centering
\small
\caption{Effect of CFG scale on CIFAR-10 and ImageNet-1k.}
\label{tab:cfg_scale_results}

\setlength{\tabcolsep}{2.5pt}
\begin{tabular}{c cccc c}
\toprule
& \multicolumn{4}{c}{CIFAR-10} & ImageNet-1k \\
\cmidrule(lr){2-5} \cmidrule(lr){6-6}
$w$
& FID $\downarrow$ & sFID $\downarrow$ & IS $\uparrow$ & Acc (\%) $\uparrow$
& FID $\downarrow$ \\
\midrule
1.0 & 4.14 & 0.0026 & 9.24 & 83.68 & 22.49 \\
1.2 & \cellcolor{gray!15}\textbf{3.85} & 0.0024 & \cellcolor{gray!15}\textbf{9.27} & 84.47 & 21.47 \\
1.5 & 4.10 & \cellcolor{gray!15}\textbf{0.0022} & \cellcolor{gray!15}\textbf{9.27} & 85.39 & 20.43 \\
2.0 &  6.36  &  0.0029    & 9.34  & \cellcolor{gray!15}\textbf{85.82}   & \cellcolor{gray!15}\textbf{19.81} \\
\bottomrule
\end{tabular}
\end{wraptable}
We analyze the effect of guidance scale $w$ under trajectory-level CFG induced by the combination of CFM and auxiliary distributions. 
The model is trained using Algorithm~\ref{alg:auxpath_cond_train} and sampled using Algorithm~\ref{alg:auxpath_cond_cfg}. 

As summarized in Table~\ref{tab:cfg_scale_results}, trajectory-level CFG derived from auxiliary paths exhibits behavior consistent with standard classifier-free guidance: increasing $w$ improves semantic alignment while maintaining competitive generative quality. 
Notably, this formulation preserves the effectiveness of CFG without requiring multiple evaluations of the backbone model, reducing the computational cost by approximately half.

At the same time, the underlying properties of CFG are retained, with $w$ controlling the trade-off between fidelity and semantic consistency in a similar manner to conventional approaches. 
Additional experimental results and analysis are provided in Appendix~\ref{sec:auxpath_finetune}.

\section{Conclusion}

This work establishes AuxPath-FM as a versatile generative modeling framework that accommodates arbitrary auxiliary distributions within the probability path. By encoding structured semantic information into these auxiliary components, the training objective achieves higher precision during velocity field parameterizing. Notably, this formulation enables trajectory-level classifier-free guidance (CFG) using only a single backbone evaluation, halving the inference cost compared to standard methods. Extensive experiments on MNIST, CIFAR-10, and ImageNet-1k validate that semantic-auxiliary CFG maintains superior generative quality while significantly enhancing computational efficiency.

\bibliographystyle{plainnat}  
\small
\bibliography{Reference}
\normalsize

\newpage
\appendix

\section{Proof of AuxPath-FM Consistency}
\label{sec:AuxPath_model}

We present an AuxPath formulation of AuxPath-FM that introduces an auxiliary variable into the probability path. Let $X_1 \sim p_{\mathrm{data}}$ denote data samples, $X_0 \sim p_{\mathrm{init}}$ denote base samples, and $\eta \sim p_\eta$ denote an auxiliary variable drawn from an arbitrary distribution. The probability path is defined as
\begin{equation}
X_t = a(t)X_1 + b(t)X_0 + c(t)\eta,
\end{equation}
where $a(t), b(t), c(t)$ satisfy boundary conditions
$a(0)=0$, $a(1)=1$, $b(0)=1$, $b(1)=0$, and $c(0)=c(1)=0$.

\subsection{Continuity Equation}

We provide a full derivation of the continuity equation under the AuxPath-FM formulation.

\subsubsection{Conditional dynamics}

Conditioned on $(X_0, \eta)$, the trajectory induces a conditional density $p_t(x \mid x_0, \eta)$ governed by the continuity equation:
\begin{equation}
\partial_t p_t(x \mid x_0, \eta)
=
-\nabla_x \cdot \left(p_t(x \mid x_0, \eta)\, u_t(x \mid x_0, \eta)\right),
\label{eq:cond_continuity_clean}
\end{equation}
where $u_t(x \mid x_0, \eta)$ denotes the conditional velocity field induced by the path
$X_t = a(t)X_1 + b(t)X_0 + c(t)\eta$.

\subsubsection{Marginal distribution}

The marginal density is obtained by integrating over the joint distribution of $(X_0, \eta)$:
\begin{equation}
p_t(x)
=
\iint p_t(x \mid x_0, \eta)\,
p_{\mathrm{init}}(x_0)\,
p_\eta(\eta)\,
dx_0\, d\eta.
\label{eq:marginal_clean}
\end{equation}

We define the corresponding marginal velocity field as
\begin{equation}
u_t(x)
=
\frac{
\iint p_t(x \mid x_0, \eta)\,
u_t(x \mid x_0, \eta)\,
p_{\mathrm{init}}(x_0)\,
p_\eta(\eta)\,
dx_0\, d\eta
}{p_t(x)}.
\label{eq:marginal_vector_clean}
\end{equation}

\subsubsection{Proof of the continuity equation}

\begin{proof}
We start by differentiating the marginal density:
\begin{align*}
\partial_t p_t(x)
&=
\partial_t \iint p_t(x \mid x_0, \eta)\,
p_{\mathrm{init}}(x_0)\,
p_\eta(\eta)\,
dx_0\, d\eta \\
&=
\iint \partial_t p_t(x \mid x_0, \eta)\,
p_{\mathrm{init}}(x_0)\,
p_\eta(\eta)\,
dx_0\, d\eta.
\end{align*}

Substituting the conditional continuity equation \eqref{eq:cond_continuity_clean}, we obtain:
\begin{align*}
\partial_t p_t(x)
&=
\iint
-\nabla_x \cdot \left(p_t(x \mid x_0, \eta)\,
u_t(x \mid x_0, \eta)\right)
p_{\mathrm{init}}(x_0)\,
p_\eta(\eta)\,
dx_0\, d\eta \\
&=
-\nabla_x \cdot
\left(
\iint p_t(x \mid x_0, \eta)\,
u_t(x \mid x_0, \eta)\,
p_{\mathrm{init}}(x_0)\,
p_\eta(\eta)\,
dx_0\, d\eta
\right).
\end{align*}

Using the definition of the marginal velocity field \eqref{eq:marginal_vector_clean}, we rewrite:
\begin{align*}
\iint p_t(x \mid x_0, \eta)\,
u_t(x \mid x_0, \eta)\,
p_{\mathrm{init}}(x_0)\,
p_\eta(\eta)\,
dx_0\, d\eta
=
p_t(x)\, u_t(x).
\end{align*}

Substituting back yields:
\begin{equation*}
\partial_t p_t(x)
=
-\nabla_x \cdot \bigl(p_t(x)\, u_t(x)\bigr),
\end{equation*}
which completes the proof.
\end{proof}

\section{AuxPath Flow Matching Loss Functions}
\label{sec:AuxPath_loss}

\subsection{Definitions}

We define two equivalent training objectives under AuxPath-FM:

\begin{itemize}
\item \textbf{AuxPath Conditional Flow Matching (AuxPath):}
\begin{equation}
\mathcal{L}_{\mathrm{AuxPath}}(\theta)
=
\mathbb{E}_{\substack{
t \sim \mathcal{U}[0,1]\\
X_0 \sim p_{\mathrm{init}}\\
X_1 \sim p_{\mathrm{data}}\\
\eta \sim p_\eta}}
\left\|
v_\theta(X_t,t)
-
\left(\dot{a}(t)X_1 + \dot{b}(t)X_0 + \dot{c}(t)\eta\right)
\right\|^2,
\end{equation}
where $X_t = a(t)X_1 + b(t)X_0 + c(t)\eta$.

\item \textbf{AuxPath Flow Matching (AuxPath-FM):}
\begin{equation}
\mathcal{L}_{\mathrm{AuxPath}}(\theta)
=
\mathbb{E}_{t \sim \mathcal{U}[0,1],\, x \sim p_t}
\left\|
v_\theta(x,t) - u_t^{\mathrm{target}}(x)
\right\|^2.
\end{equation}
\end{itemize}

\subsection{Loss Function Equivalence Theorem}
We define the conditional flow matching loss as $\mathcal{L}_{\mathrm{AuxPath}}(\theta) = \mathbb{E}_{t, X_1, X_0, \eta} [ \| v_\theta(X_t, t) - \dot{X}_t \|^2 ]$ and the marginal flow matching loss as $\mathcal{L}_{\mathrm{AuxPath\text{-}Marginal}}(\theta) = \mathbb{E}_{t, x \sim p_t} [ \| v_\theta(x, t) - u_t(x) \|^2 ]$.

\begin{theorem}
The conditional objective is equivalent to the marginal objective up to a constant $C$ independent of $\theta$:
\begin{equation}
    \mathcal{L}_{\mathrm{AuxPath\text{-}Marginal}}(\theta)
    =
    \mathcal{L}_{\mathrm{AuxPath}}(\theta) + C.
\end{equation}
\end{theorem}

\begin{proof}

We start from the marginal objective:
\begin{equation}
\mathcal{L}_{\mathrm{AuxPath\text{-}Marginal}}(\theta)
=
\mathbb{E}_{t,x \sim p_t}
\left[
\| v_\theta(x,t) - u_t(x) \|^2
\right].
\end{equation}

Expanding the squared norm:
\begin{align}
\mathcal{L}_{\mathrm{AuxPath\text{-}Marginal}}(\theta)
&=
\mathbb{E}_{t,x}
\Big[
\|v_\theta(x,t)\|^2
- 2\, v_\theta(x,t)^T u_t(x)
+ \|u_t(x)\|^2
\Big].
\end{align}

Using the definition
\begin{equation}
u_t(x) = \mathbb{E}[\dot{X}_t \mid x],
\end{equation}
we apply the law of total expectation:
\begin{align}
\mathbb{E}_{t,x}
\big[
v_\theta(x,t)^T u_t(x)
\big]
&=
\mathbb{E}_{t,x}
\Big[
v_\theta(x,t)^T \mathbb{E}[\dot{X}_t \mid x]
\Big] \\
&=
\mathbb{E}_{t,X_t}
\big[
v_\theta(X_t,t)^T \dot{X}_t
\big].
\end{align}

Similarly,
\begin{equation}
\mathbb{E}_{t,x}\|u_t(x)\|^2
=
\mathbb{E}_{t,X_t}
\Big[
\|\mathbb{E}[\dot{X}_t \mid X_t]\|^2
\Big],
\end{equation}
which is independent of $\theta$.

Therefore,
\begin{align}
\mathcal{L}_{\mathrm{AuxPath\text{-}Marginal}}(\theta)
&=
\mathbb{E}_{t,X_t}
\Big[
\|v_\theta(X_t,t)\|^2
- 2\, v_\theta(X_t,t)^T \dot{X}_t
\Big]
+ C_1.
\end{align}

On the other hand, expanding the conditional objective:
\begin{align}
\mathcal{L}_{\mathrm{AuxPath}}(\theta)
&=
\mathbb{E}_{t,X_t}
\Big[
\|v_\theta(X_t,t)\|^2
- 2\, v_\theta(X_t,t)^T \dot{X}_t
+ \|\dot{X}_t\|^2
\Big].
\end{align}

Comparing the two expressions, we obtain:
\begin{equation}
\mathcal{L}_{\mathrm{AuxPath\text{-}Marginal}}(\theta)
=
\mathcal{L}_{\mathrm{AuxPath}}(\theta)
+ C,
\end{equation}
where $C = C_1 - \mathbb{E}[\|\dot{X}_t\|^2]$ is independent of $\theta$.

\end{proof}

\section{Compositional and Degenerate Auxiliary Variables}

\subsection{Compositional Auxiliary Distributions}

The auxiliary variable $\eta$ can be constructed as a mixture:
\begin{equation}
\eta \sim \sum_k \pi_k p_\eta^{(k)}, \qquad \sum_k \pi_k = 1.
\end{equation}
This remains within the AuxPath-FM class and induces valid probability paths.

\subsection{Deterministic Auxiliary Variables}
\label{sec:Y_is_FX}

We further consider the case where $\eta$ is a deterministic function of endpoints:
\begin{equation}
\eta = F(X_0).
\end{equation}

The path becomes
\begin{equation}
X_t = a(t)X_1 + b(t)X_0 + c(t)F(X_0).
\end{equation}

Conditioned on $X_0$, the dynamics reduce to
\begin{equation}
\partial_t p_t(x|x_0)
=
-\nabla_x \cdot (p_t(x|x_0) u_t(x|x_0)).
\end{equation}

The marginal distribution becomes
\begin{equation}
p_t(x) = \int p_t(x|x_0)p_{\mathrm{init}}(x_0)dx_0,
\end{equation}
with velocity field
\begin{equation}
u_t(x) =
\frac{1}{p_t(x)}
\int p_t(x|x_0)u_t(x|x_0)p_{\mathrm{init}}(x_0)dx_0.
\end{equation}

Thus the formulation reduces to standard flow matching.

\section{Trajectory-Level Classifier-Free Guidance}
\label{sec:auxpath_cfg}

AuxPath-FM naturally induces a trajectory-level formulation of classifier-free guidance (CFG)~\cite{ho2022classifier} by encoding semantic information directly into the auxiliary variable $\eta$.

\paragraph{Auxiliary variable parameterization.}
We consider a structured auxiliary variable of the form
\begin{equation}
    \eta = F_\phi(y),
\end{equation}
where $y$ denotes the condition (e.g., class label), and $F_\phi$ is a lightweight neural network mapping labels to feature vectors.

Under the auxiliary probability path
\begin{equation}
    X_t = a(t)X_1 + b(t)X_0 + c(t)\eta,
\end{equation}
the time derivative becomes
\begin{equation}
    \dot{X}_t = \dot{a}(t)X_1 + \dot{b}(t)X_0 + \dot{c}(t)\eta.
\end{equation}

\paragraph{Induced conditional velocity field.}
From the flow matching objective, the learned velocity field can be decomposed as
\begin{equation}
    v(x,t,y) \approx v_\theta(x,t) + \dot{c}(t) F_\phi(y),
\end{equation}
where $v_\theta(x,t)$ captures the base transport dynamics, and $\dot{c}(t)F_\phi(y)$ acts as a condition-dependent drift term induced by the auxiliary path.

Similarly, the unconditional velocity field is obtained by replacing $y$ with the null condition $\emptyset$:
\begin{equation}
    v(x,t,\emptyset) \approx v_\theta(x,t) + \dot{c}(t) F_\phi(\emptyset).
\end{equation}

\paragraph{Training of $F_\phi$.}
We train $F_\phi$ to learn semantic prototypes corresponding to class-wise centers. Specifically, we minimize
\begin{equation}
    \mathcal{L}_{F} =
    \mathbb{E}_{X_1, y}
    \left[
        \| F_\phi(y) - X_1 \|^2
    \right],
\end{equation}
which encourages
\begin{equation}
    F_\phi(y) \approx \mathbb{E}[X_1 \mid y].
\end{equation}

For the unconditional embedding, we define
\begin{equation}
    F_\phi(\emptyset) \approx \mathbb{E}[X_1],
\end{equation}
which corresponds to the global data centroid.

\paragraph{Trajectory-level CFG.}
Applying standard classifier-free guidance yields
\begin{equation}
\begin{aligned}
    v_{\mathrm{cfg}}(x,t)
    &= v(x,t,\emptyset)
    + w \big( v(x,t,y) - v(x,t,\emptyset) \big) \\
    &= v_\theta(x,t)
    + \dot{c}(t)\Big[
        F_\phi(\emptyset)
        + w\big(F_\phi(y) - F_\phi(\emptyset)\big)
    \Big].
\end{aligned}
\end{equation}

\paragraph{Interpretation.}
The guidance term corresponds to a linear interpolation in the auxiliary space:
\begin{equation}
    \eta_{\mathrm{cfg}} =
    F_\phi(\emptyset) + w\big(F_\phi(y) - F_\phi(\emptyset)\big),
\end{equation}
which shifts the trajectory from the global centroid toward the class-specific centroid.

\paragraph{Efficiency.}
Unlike standard CFG, which requires two evaluations of $v_\theta(x,t)$, the above formulation reuses a single evaluation of $v_\theta$ and applies guidance only through the lightweight module $F_\phi$, requiring only two evaluations of $F_\phi$.

\paragraph{Implication.}
This shows that classifier-free guidance can be interpreted as a \emph{trajectory-level perturbation} in the auxiliary space, providing a more efficient and principled alternative to conventional CFG.

\subsection{Compatibility with Dual CFG Scales}
\label{subsec:cfg_dual}

We further study the compatibility between trajectory-level CFG and standard classifier-free guidance (CFG) under different guidance scales. All experiments use a \textbf{DiT-B/2} backbone~\cite{peebles2023dit}.

We consider two independent guidance strengths: the trajectory-level scale $w_{\text{aux}}$ and the standard CFG scale $w_{\text{CFG}}$, both selected from $\{1.0, 1.2\}$, yielding $2 \times 2$ configurations.

A key difference lies in computation cost. Standard CFG requires \textbf{two backbone evaluations per sampling step} (conditional and unconditional forward passes), while trajectory-level CFG modifies the velocity via a lightweight auxiliary term and keeps the backbone evaluation \textbf{single-pass}. Accordingly, we report:

\begin{itemize}
    \item \textbf{Pass}: number of backbone evaluations per denoising step (1 for single-pass, 2 for standard CFG);
    \item \textbf{GFLOPs}: total computational cost per image generation (50 sampling steps).
\end{itemize}

Results on ImageNet-1k ($256 \times 256$) are reported in Table~\ref{tab:cfg_dual_dit}. We observe stable behavior across all configurations. 

Overall, trajectory-level CFG remains fully compatible with standard CFG, while consistently reducing inference cost due to the single-pass backbone evaluation.

\begin{table}[t]
\centering
\caption{Joint evaluation of trajectory-level CFG and standard CFG on ImageNet-1k using DiT-B/2. Pass denotes backbone evaluations per step. GFLOPs are measured per image generation (50 steps).}
\label{tab:cfg_dual_dit}

\setlength{\tabcolsep}{3.2pt}
\small
\begin{tabular}{c c c c c}
\toprule
$w_{\text{aux}}$ & $w_{\text{CFG}}$ & Pass & GFLOPs & FID $\downarrow$ \\
\midrule
1.0 & 1.0 & 1 & 2301.6028 & 22.49 \\
1.0 & 1.2 & 2 & 2303.0808 & 12.75 \\
\midrule
1.2 & 1.0 & 1 & 4601.7198 & 21.47 \\
1.2 & 1.2 & 2 & 4603.1978 & 11.87 \\
\bottomrule
\end{tabular}
\end{table}

\section{From Unconditional to Conditional Generation via Auxiliary Paths}
\label{sec:auxpath_finetune}

A key advantage of AuxPath-FM is that it enables converting a pre-trained unconditional flow model into a conditional generator with minimal additional cost.

\paragraph{Method.}
Given a pre-trained velocity model $v_\theta(x,t)$ trained on unconditional data, we introduce a lightweight semantic module $F_\phi(y)$ and fine-tune the model using the conditional AuxPath formulation (Algorithm~\ref{alg:auxpath_cond_train}). 
Importantly, the velocity model is trained to match the baseline flow field, while the auxiliary term $\eta = F_\phi(y)$ is incorporated only through the probability path. 
At inference time, conditional generation is achieved by augmenting the dynamics as
\begin{equation}
    v(x,t,y) = v_\theta(x,t) + \dot{c}(t) F_\phi(y),
\end{equation}
without modifying the original model architecture.

\paragraph{Results.}
We evaluate this strategy by initializing from a pre-trained unconditional model and fine-tuning with the auxiliary module. 
Results in Table~\ref{tab:finetune_results} show that AuxPath-FM effectively enables conditional generation with strong performance, demonstrating that the auxiliary path provides a simple yet effective mechanism for upgrading existing models.

\begin{table}[h]
\centering
\caption{Unconditional-to-conditional fine-tuning results.}
\begin{tabular}{lcccc}
\toprule
Method & Acc(\%) $\uparrow$ & FID $\downarrow$ & sFID $\downarrow$ & IS $\uparrow$\\
\midrule
Unconditional  &  N/A & 3.72 & 0.0017& 9.06\\
 + numeric $\eta$ & 79.18 & 5.73 & 0.0038 & 8.89\\
+ learned $\eta$ & 77.03 & 5.90 & 0.0043 & 8.93\\
\bottomrule
\end{tabular}
\label{tab:finetune_results}
\end{table}

\paragraph{Discussion.}
These results highlight that AuxPath-FM provides a lightweight and effective mechanism for transforming unconditional generative models into conditional ones, without requiring architectural changes or retraining from scratch.

\begin{figure*}[h]
  \centering
  \scalebox{0.95}{
  \begin{minipage}{0.48\textwidth}
    \centering
    \includegraphics[width=\linewidth]{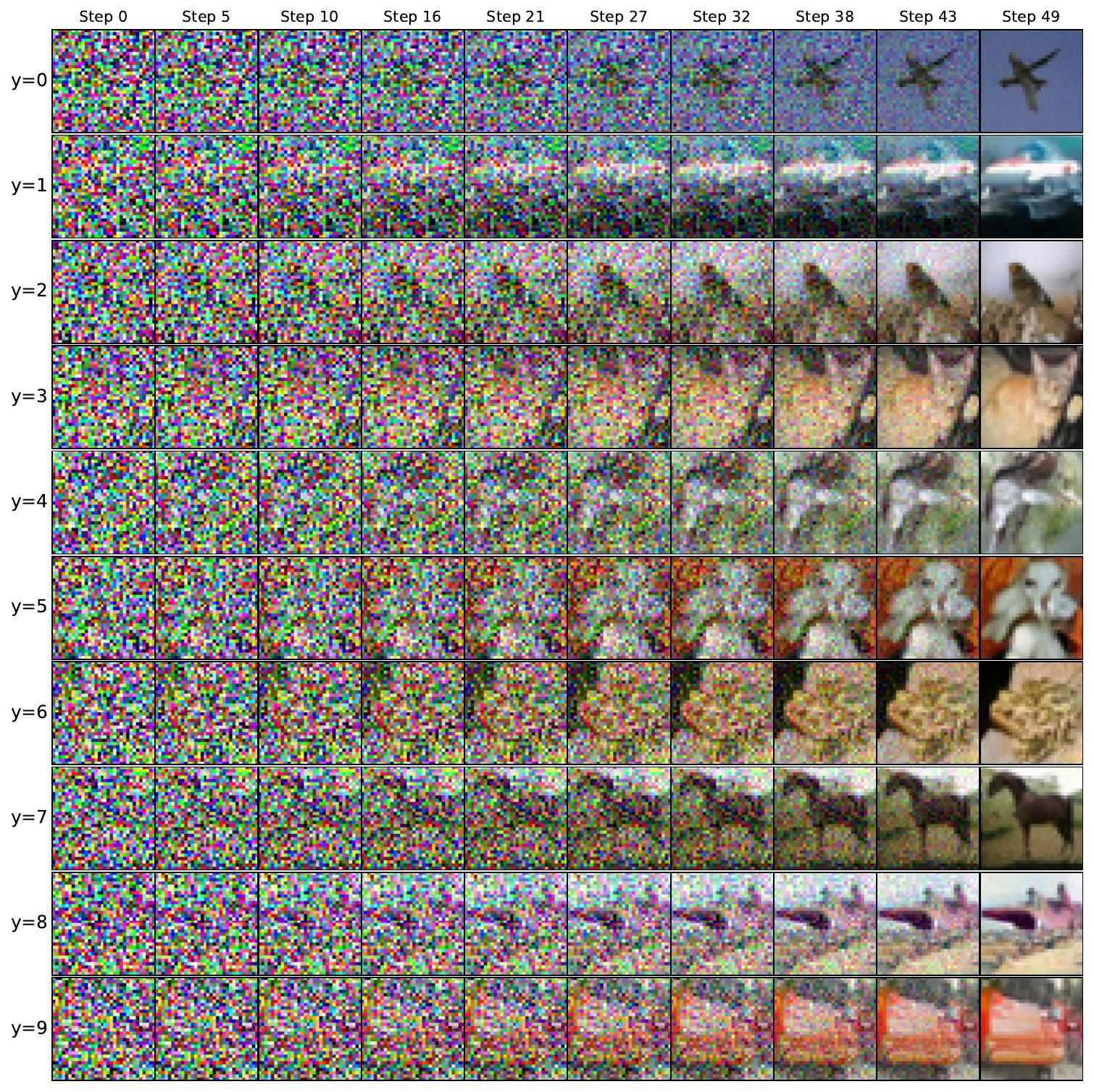}
    \small (a) Trajectory with $w=1.0$
  \end{minipage}
  \hfill
  \begin{minipage}{0.48\textwidth}
    \centering
    \includegraphics[width=\linewidth]{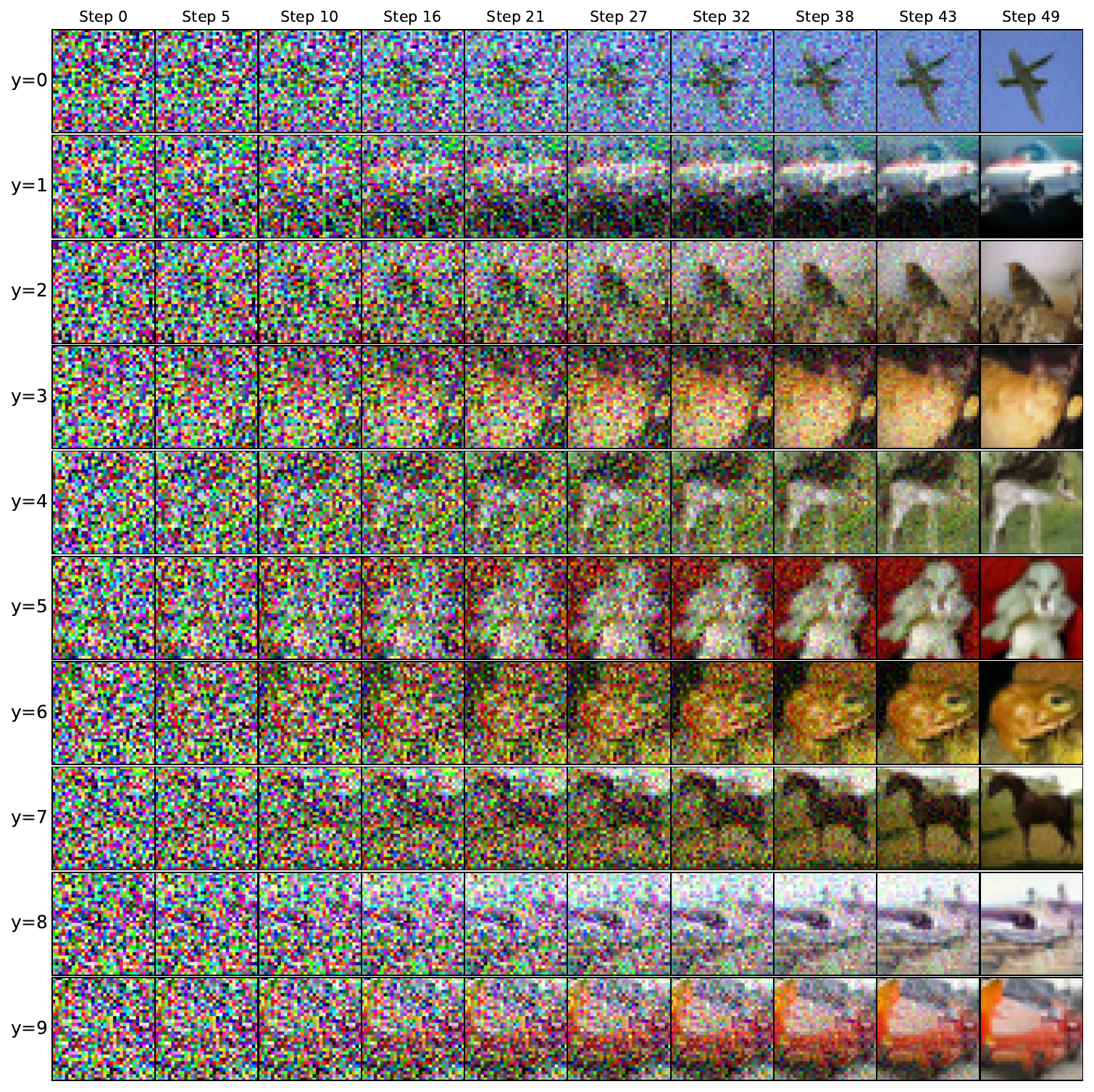}
    \small (b) Trajectory with $w=2.0$
  \end{minipage}}
  \caption{\textbf{Generation trajectories of AuxPath-FM with trajectory-level CFG on CIFAR-10.}
  Each row shows the evolution of a sample along the probability path from the initial state ($t=0$) to the final generated image ($t=1$). 
  Unlike methods that apply guidance through repeated network evaluations, AuxPath-FM incorporates classifier-free guidance directly into the trajectory via auxiliary variables. 
  Increasing the guidance scale from $w=1.0$ to $w=2.0$ results in more semantically aligned samples, while preserving smooth and coherent transitions along the path. 
  This demonstrates that trajectory-level guidance effectively steers generation throughout the entire flow using only a single backbone evaluation.}
  \label{fig:cifar10_auxpath_trajs}
\end{figure*}

\begin{figure*}[h]
  \centering
    \scalebox{0.9}{
  \begin{minipage}{0.48\textwidth}
    \centering
    \includegraphics[width=\linewidth]{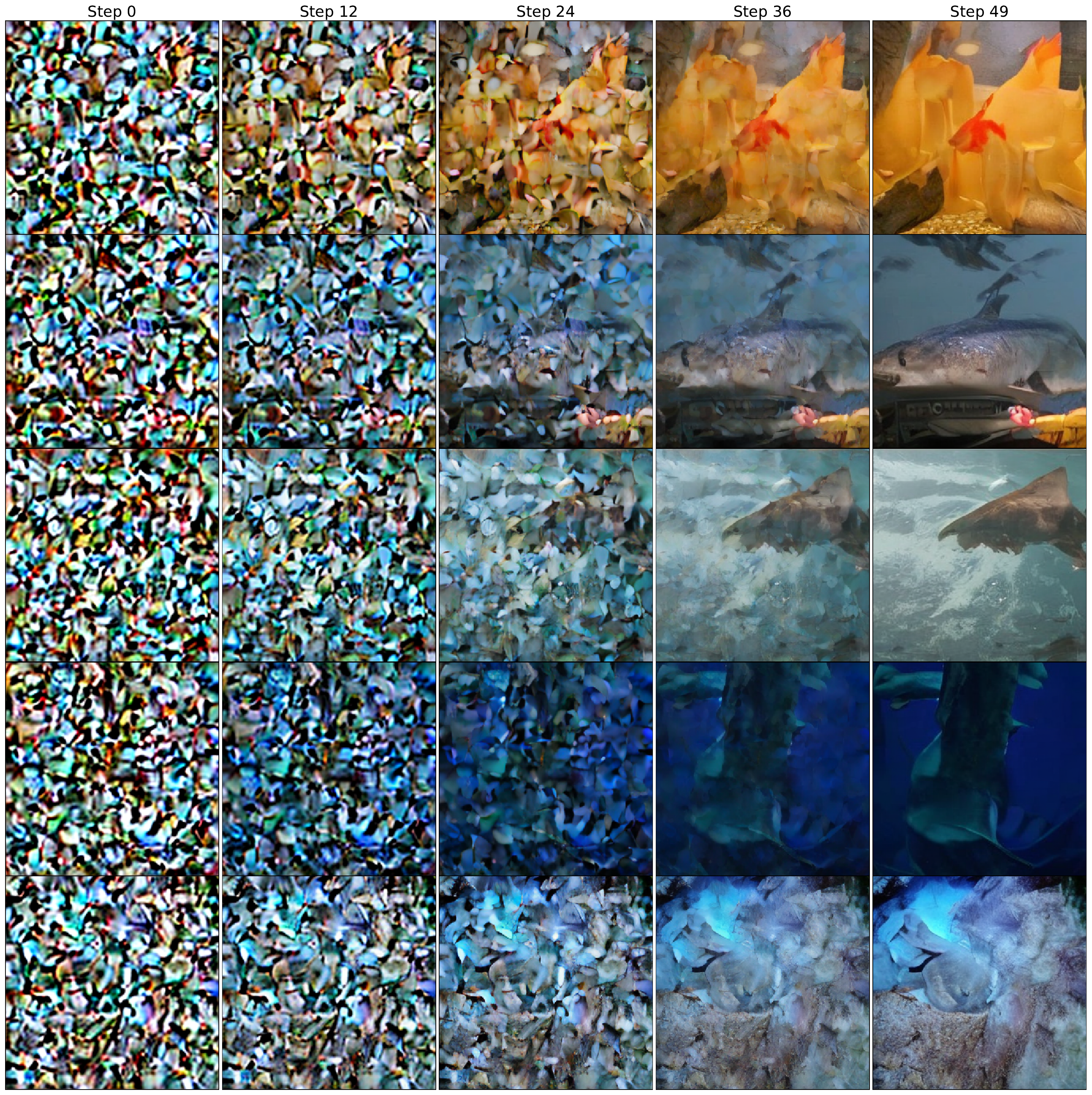}
    \small (a) $w=1.0$
  \end{minipage}
  \hfill
  \begin{minipage}{0.48\textwidth}
    \centering
    \includegraphics[width=\linewidth]{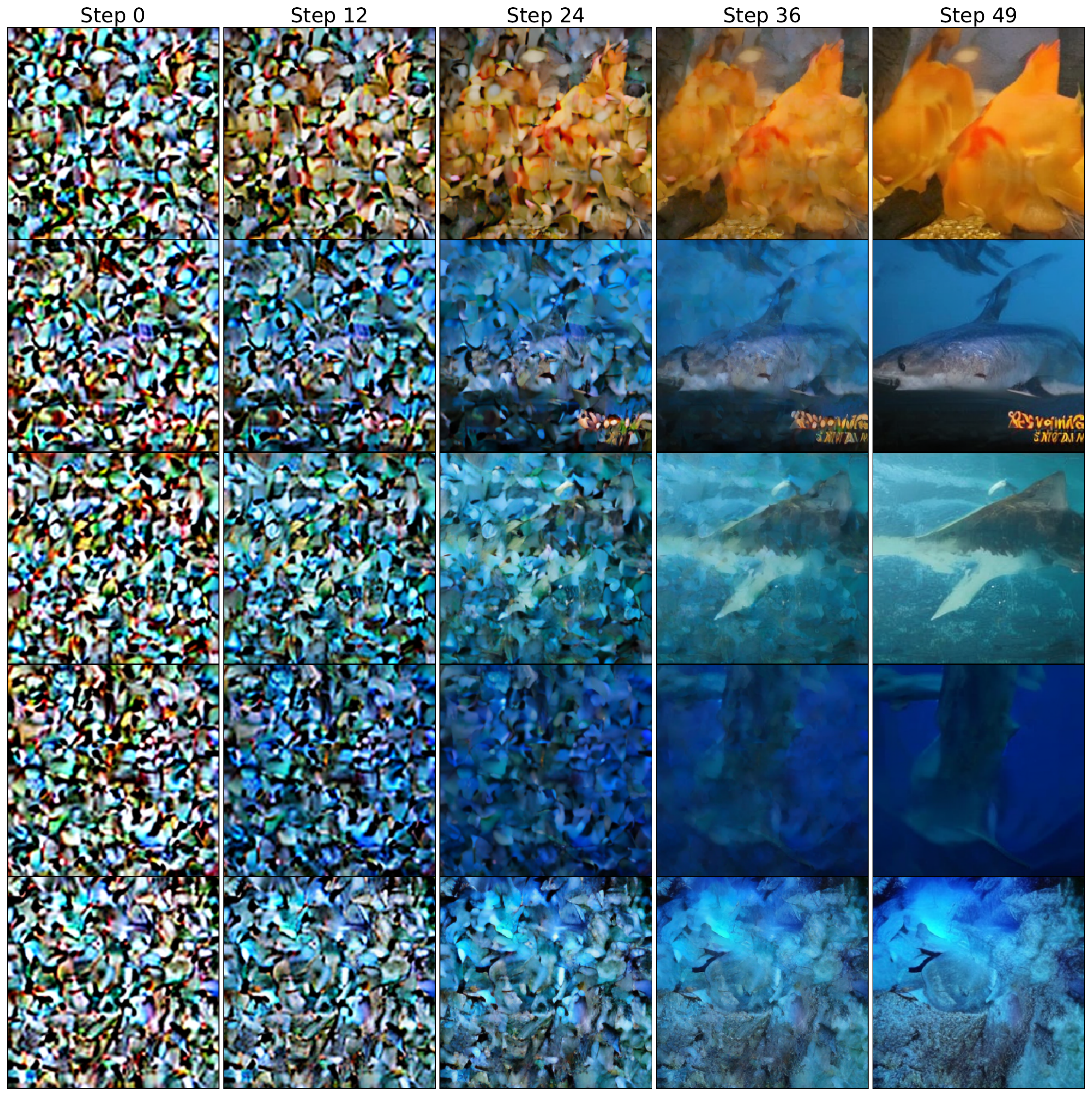}
    \small (b) $w=1.5$
  \end{minipage}}

  \caption{\textbf{Trajectory-level CFG on ImageNet-1k with AuxPath-FM.}
  Each row shows the evolution of a sample along the probability path from $t=0$ to $t=1$. 
  Increasing the guidance scale from $w=1.0$ to $w=1.5$ yields more semantically aligned samples while preserving smooth trajectory transitions. 
  Guidance is applied directly along the path via auxiliary variables, enabling effective control with a single backbone evaluation.}
  
  \label{fig:imagenet_auxpath_cfg}
\end{figure*}

\section{Additional Qualitative Results}
\label{appendix:qualitative}

In this section, we present additional qualitative results to visualize the generation process of AuxPath-FM under different guidance strengths. These results complement the quantitative evaluations in the main text by illustrating how semantic auxiliary guidance influences the entire sampling trajectory.

\section{Anonymous Code Release and Reproducibility}
\label{sec:appendix_code}

To facilitate reproducibility and enable further research, we provide an anonymous code release of AuxPath Flow Matching at:

\begin{center}
\url{https://github.com/XinPeng76/AuxPath-FM.git}
\end{center}

\paragraph{Code scope.}
The current anonymous repository includes the core implementation of the proposed AuxPath framework, with a particular focus on auxiliary variable modeling and flexible path construction via different choices of terminal noise $\eta$. The released code covers the training logic used in ImageNet-1000 experiments (see \texttt{train\_imagenet256\_DIT\_AFM.py}).

\section{Statement on the Use of Large Language Models}
\label{app:llm_statement}

During the preparation of this manuscript, large language models (LLMs) were used in a limited manner
solely for language editing purposes, such as improving clarity, grammar, and academic style.
All aspects of the research conception, methodological development, experimental design,
analysis of results, and the scientific conclusions presented in this paper
were carried out independently by the authors.


\end{document}

%% file: code.tex
\definecolor{codeblue}{rgb}{0.25,0.5,0.5}
\definecolor{codekw}{rgb}{0.85, 0.18, 0.50}

\definecolor{codesign}{RGB}{0, 0, 255}
\definecolor{codefunc}{rgb}{0.85, 0.18, 0.50}

\lstdefinelanguage{PythonFuncColor}{
  language=Python,
  keywordstyle=\color{blue}\bfseries,
  commentstyle=\color{codeblue},  
  stringstyle=\color{orange},
  showstringspaces=false,
  basicstyle=\ttfamily\small,
  literate=
    {*}{{\color{codesign}* }}{1}
    {-}{{\color{codesign}- }}{1}
    {+}{{\color{codesign}+ }}{1}
    {dataloader}{{\color{codefunc}dataloader}}{1}
    {sample_t_r}{{\color{codefunc}sample\_t\_r}}{1}
    {randn}{{\color{codefunc}randn}}{1}
    {randn_like}{{\color{codefunc}randn\_like}}{1}
    {jvp}{{\color{codefunc}jvp}}{1}
    {stopgrad}{{\color{codefunc}stopgrad}}{1}
    {metric}{{\color{codefunc}metric}}{1}
}

\lstset{
  language=PythonFuncColor,
  backgroundcolor=\color{white},
  basicstyle=\fontsize{9pt}{9.9pt}\ttfamily\selectfont,
  columns=fullflexible,
  breaklines=true,
  captionpos=b,
}

\begin{wrapfigure}{r}{0.5\linewidth}
\vspace{-1.2em}
\centering
\begin{minipage}{0.95\linewidth}

\begin{algorithm}[H]
\caption{{AuxPath-FM}: Training}
\label{alg:auxpath_train}
\footnotesize
\setlength{\baselineskip}{8.8pt}

\begin{lstlisting}
# v_theta: velocity model
# a(t), b(t), c(t): path coefficients
# eta ~ p_eta (arbitrary distribution)
x0, x1 = sample_batch()
eta = sample_eta()
t = sample_uniform()

X_t = a(t)*x1 + b(t)*x0 + c(t)*eta
v_t = dot_a(t)*x1 + dot_b(t)*x0 + dot_c(t)*eta

loss = metric(v_theta(X_t,t) - v_t)
\end{lstlisting}

\end{algorithm}

\vspace{-1.4em}

\begin{algorithm}[H]
\caption{{AuxPath-FM}: Sampling}
\label{alg:auxpath_sample}
\footnotesize
\setlength{\baselineskip}{8.8pt}

\begin{lstlisting}
x = randn(x_shape)

for t in schedule:
    x = x + v_theta(x,t) * dt
\end{lstlisting}

\end{algorithm}

\end{minipage}
\vspace{-1.2em}
\end{wrapfigure}

%% file: code1.tex
\definecolor{codeblue}{rgb}{0.25,0.5,0.5}
\definecolor{codekw}{rgb}{0.85, 0.18, 0.50}

\definecolor{codesign}{RGB}{0, 0, 255}
\definecolor{codefunc}{rgb}{0.85, 0.18, 0.50}

\lstdefinelanguage{PythonFuncColor}{
  language=Python,
  keywordstyle=\color{blue}\bfseries,
  commentstyle=\color{codeblue},  
  stringstyle=\color{orange},
  showstringspaces=false,
  basicstyle=\ttfamily\small,
  literate=
    {*}{{\color{codesign}* }}{1}
    {-}{{\color{codesign}- }}{1}
    {+}{{\color{codesign}+ }}{1}
    {dataloader}{{\color{codefunc}dataloader}}{1}
    {sample_t_r}{{\color{codefunc}sample\_t\_r}}{1}
    {randn}{{\color{codefunc}randn}}{1}
    {randn_like}{{\color{codefunc}randn\_like}}{1}
    {jvp}{{\color{codefunc}jvp}}{1}
    {stopgrad}{{\color{codefunc}stopgrad}}{1}
    {metric}{{\color{codefunc}metric}}{1}
}

\lstset{
  language=PythonFuncColor,
  backgroundcolor=\color{white},
  basicstyle=\fontsize{9pt}{9.9pt}\ttfamily\selectfont,
  columns=fullflexible,
  breaklines=true,
  captionpos=b,
}

\begin{wrapfigure}{r}{0.5\linewidth}
\vspace{-1.2em}
\centering
\begin{minipage}{0.95\linewidth}

\begin{algorithm}[H]
\caption{{AuxPath-FM (Cond.)}: Training}
\label{alg:auxpath_cond_train}
\footnotesize
\setlength{\baselineskip}{8.8pt}

\begin{lstlisting}
# Stage 1: train semantic model
x1, y = sample_data()     
pred = F_phi(y)
loss_F = metric(pred - x1)

# Stage 2: train flow model
x0, x1, y = sample_batch()  
eta = F_phi(y)          
t = sample_uniform()
X_t = a(t)*x1 + b(t)*x0 + c(t)*eta
v_t = dot_a(t)*x1 + dot_b(t)*x0
loss = metric(v_theta(X_t,t) - v_t)
\end{lstlisting}

\end{algorithm}

\vspace{-1.4em}

\begin{algorithm}[H]
\caption{{AuxPath-FM (Cond.)}: Sampling}
\label{alg:auxpath_cond_sample}
\footnotesize
\setlength{\baselineskip}{8.8pt}

\begin{lstlisting}
x = randn(x_shape)
eta = F_phi(y)
for t in schedule:
    v = v_theta(x, t) + dot_c(t)*eta
    x = x + v * dt
\end{lstlisting}

\end{algorithm}

\vspace{-1.4em}

\begin{algorithm}[H]
\caption{{AuxPath-FM}: CFG Sampling}
\label{alg:auxpath_cond_cfg}
\footnotesize
\setlength{\baselineskip}{8.8pt}

\begin{lstlisting}
x = randn(x_shape)
eta_c = F_phi(y)
eta_u = F_phi(empty)
eta_cfg = eta_u + w * (eta_c - eta_u)

for t in schedule:
    # trajectory-level velocity
    v = v_theta(x, t) + dot_c(t) * eta_cfg
    x = x + v * dt
\end{lstlisting}

\end{algorithm}

\end{minipage}
\vspace{-1.2em}
\end{wrapfigure}